\documentclass[10pt,conference,compsocconf]{IEEEtran}
\usepackage{algorithm}
\usepackage[noend]{algorithmic}
\usepackage[cmex10]{amsmath}
\usepackage{booktabs}
\usepackage{graphicx}
\usepackage{subfig}
\usepackage{tabularx}

\newcommand{\calX}{\mathcal{X}}
\newcommand{\SABNA}{\textsf{\small SABNA}}
\newcommand{\UrL}{\textsf{\small URLearning}}

\newtheorem{theorem}{Theorem}

\begin{document}
\title{Exact Structure Learning of Bayesian Networks by Optimal Path Extension}

\author{\IEEEauthorblockN{Subhadeep Karan}
  \IEEEauthorblockA{Department of Computer Science and Engineering\\
    University at Buffalo\\
    Buffalo, NY, USA\\
    Email: skaran@buffalo.edu}
  \and
  \IEEEauthorblockN{Jaroslaw Zola}
  \IEEEauthorblockA{
    Department of Computer Science and Engineering\\
    Department of Biomedical Informatics\\
    University at Buffalo\\
    Buffalo, NY, USA\\
    Email: jzola@buffalo.edu}
}
\maketitle

\begin{abstract}
Bayesian networks are probabilistic graphical models often used in big data analytics. The problem of Bayesian network exact structure learning is to find a network structure that is optimal under certain scoring criteria. The problem is known to be NP-hard and the existing methods are both computationally and memory intensive. In this paper, we introduce a new approach for exact structure learning that leverages relationship between a partial network structure and the remaining variables to constrain the number of ways in which the partial network can be optimally extended. Via experimental results, we show that the method provides up to three times improvement in runtime, and orders of magnitude reduction in memory consumption over the current best algorithms.
\end{abstract}

\begin{IEEEkeywords}
  Bayesian Networks; Exact Structure Learning; Score-based Learning;
\end{IEEEkeywords}

\IEEEpeerreviewmaketitle

\section{Introduction}

Bayesian networks (BNs) are a class of probabilistic graphical models that capture conditional relationships among a set of random variables. In BNs, the relationship between variables is qualitatively described by conditional independencies, and quantitatively assessed by conditional probability distributions. BNs serve as a powerful tool for structuring probabilistic information and hence are an ideal framework for complex inferences including predictive, diagnostic and explanatory reasoning~\cite{Koller2009}. Over the last decades, BNs have been successfully applied in many domains ranging from diagnostic systems~\cite{Kahn1997,Bobbio2001,Pearl2014}, clinical decision support~\cite{Sesen2013,Forsberg2011,Heckerman1992}, uncertainty quantification in numerical methods~\cite{Hawkins-Daarud2013}, to systems biology~\cite{Ott2004,Needham2007,Nikolova2013} and genomics~\cite{Jiang2011}. In many real-life applications, BNs outperform more sophisticated machine learning methods, or are desired due to the support for speculative queries and the ease of interpretation. For example, in~\cite{Zhang2007} a simple BN with six variables has been applied to characterize various behaviors of service-oriented computer systems. The network was consistently more accurate in predicting the studied systems' response time when compared to fairly sophisticated neural nets. In systems biology, BNs are often built from gene expression data and are directly used to analyze potential regulatory interactions between genes, which is possible thanks to the explicit network structure representation~\cite{Needham2007}.

While BNs offer multiple advantages in how they represent probabilities and how they explicitly handle uncertainty, they also pose challenges. This is because both structure learning and probabilistic inference in BNs have been demonstrated to be NP-hard~\cite{Chickering1996,Cooper1990}. This fact becomes significant in the context of big data. On the one hand, by leveraging big data we can consider more complex and realistic networks (e.g. by including more variables), and we can obtain more accurate probability estimates required to learn and parametrize these networks. On the other hand, the computational complexity of discovering BN structure becomes prohibitive for large data and exact learning algorithms have to be replaced by approximations or heuristics~\cite{Koller2009}. However, these algorithms do not provide guarantees on the quality of the structures they find. At the same time, in many real-life scenarios finding the optimal BN structure is a necessity, for example to make different models (e.g. BNs learned under different statistical criteria) comparable, or to allow for precise reasoning about models' performance.

To address this challenge, we introduce a new method to accelerate a scoring-based exact structure learning of BNs. Our method, which we call {\it optimal path extension}, leverages shortest-path formulation of the BN structure learning. It takes advantage of the relationship between a partial network structure and the remaining variables to constraint the number of ways in which the partial network can be optimally extended. This has the effect of ``compacting'' the dynamic programming lattice explored during the structure search, thus practically reducing computational and memory complexity. The technique is general and can be combined with various BN search strategies, such as BFS or different variants of the A-star algorithm. Through experimental results, we show that the method provides up to three times improvement in runtime, and orders of magnitude reduction in the memory consumption over the current best algorithms. Thus, our approach significantly expands the range of applications in which the exact BN structure learning can be applied, including for big data analytics.

The remainder of this paper is organized as follows: in Section~\ref{sec:problem}, we provide basic definitions and formally state the BN exact structure learning problem. In Section~\ref{sec:method} we introduce our proposed method, and we demonstrate its experimental validation in Section~\ref{sec:results}. We conclude the paper in Section~\ref{sec:conclusion}.

\section{Preliminaries and Problem Formulation}\label{sec:problem}

Formally, a Bayesian network over a set of $n$ random variables $\calX = \{X_1, \ldots, X_n\}$ is a pair $(G,P)$, where $G$ is a directed acyclic graph (DAG) with a set of vertices $\calX$, $P$ is a joint probability distribution over the same set of variables, and $G$ encodes conditional independencies induced by $P$. Let $Pa(X_i)$ denote a set of parents of $X_i$ in $G$, i.e. $Pa(X_i)$ consists of all $X_j \in \calX$ such that there exists an edge from $X_j$ to $X_i$ in $G$. If the pair $(G, P)$ is a Bayesian network, then every variable $X_{i}$ must be independent from its non-descendants given its parents $Pa(X_i)$. Here non-descendants of $X_i$ are all variables that cannot be reached from $X_i$ in $G$. Intuitively, a Bayesian network provides compact and graphical representation of the joint probability~$P$. Indeed, following the chain rule of probability a Bayesian network allows for a succinct factorization of~$P$, which in turn drastically reduces the cost of inferences and enables qualitative analysis of the resulting model.

Bayesian networks can be regarded as supervised techniques in the sense that both parameters and structure of a BN can be learned from data. Given a complete input data set represented by $D = \{ D_1,\ldots,D_n\}$, where $D_i$ is a vector of $m$ observations of $X_i$, we are interested in finding a graph $G$ that best explains data in $D$. This problem is known as Bayesian network structure learning. In general, there are two broad classes of structure learning methods. In the constraint-based learning, a statistical test is used to identify a DAG that is consistent with independencies encoded by the data $D$~\cite{Neapolitan2003}. These techniques are heuristics and they offer limited theoretical guarantees with respect to the solutions they find. In the score-based learning, a search strategy is used to find a DAG that is optimal under a certain scoring criterion~\cite{Friedman1997}. Because these are optimization techniques, exact solutions can be found and reasoned about.

Let $Score(G{:}D)$ be a scoring function evaluating quality of the network structure $G$ with respect to the input data~$D$. Furthermore, let $Score(G{:}D)$ be decomposable, that is: \[ Score(G{:}D) = \sum_{X_i \in \calX}s(X_i,Pa(X_i)), \] where $s(X_i,Pa(X_i))$ is a score contribution of $X_i$ when its parents are $Pa(X_i)$. Examples of such scoring functions include popular BIC~\cite{Schwarz1978} and MDL~\cite{Schwarz1978} derived from information theory or BD~\cite{Cooper1992} and BDe~\cite{Heckerman1999} that implement Bayesian scoring criteria. In this paper, we are considering the exact score-based structure learning problem, which is to find an optimal\footnote{We consistently use ``an optimal'' and not ``the optimal'' as multiple optimal solutions may exist.} structure $G$ given a scoring function $Score(G{:}D)$. We do not focus on one particular scoring function and hence we do not discuss details of how to compute $s(X_i,Pa(X_i))$ from data, except to note that the cost of $s$ is related to the number of observations and the size of the parents set $Pa(X_i)$. However, we exploit the fact that the objective function is decomposable. Decomposability is commonly assumed to improve the search process as local changes to a network structure can be evaluated quickly. Nevertheless, the problem remains challenging owing to the super-exponential size of the search space.

\subsection{Optimization Problem}

Consider a set $\calX = \{ X_1, X_2, \ldots, X_n \}$ of $n$ random variables and a scoring function $Score(G{:}D)$ that we want to minimize. The search space of all potential network structures is super-exponential and consists of $C(n) = \sum_{i=1}^{n} (-1)^{(i+1)}\binom{n}{i}2^{i(n-i)}C(n-i)$ DAGs with $n$ nodes. However, any DAG with nodes $\calX$ can be equivalently represented via one of its topological orderings of $\calX$. A topological ordering implies that $X_i$ is always preceded by $X_j$, written as $X_j \prec X_i$, if $X_j$ is a parent of $X_i$, i.e. $X_j \in Pa(X_i)$. Let $\pi(U)$ denote a topological ordering over a set $U \subseteq \calX$. To find an optimal network structure it is sufficient to find its optimal ordering keeping track of parents assigned to each variable $X_i$. Because relative ordering of parents of $X_i$ is irrelevant and the scoring function is decomposable, we can leverage dynamic programming to constraint the search space. This general idea has been exploited in different variants, for example in~\cite{Ott2004,Nikolova2013,Koivisto2004,Singh2005}, and works as follows. Because any DAG must have at least one sink node (i.e. a node without descendants), we can first identify an optimal sink and find its optimal parents assignment (i.e. its optimal parents set). Then, we can continue with the remaining nodes recursively organizing them into an optimal structure. Because we know that a sink node has no successors, it can be placed at the end of the topological order we are building. Let $d(X_i, U), U \subseteq \calX-\{X_i\}$, be the score of selecting optimal parents of $X_i$ from $U$:
\begin{equation}\label{eq:d}
d(X_i,U) = \min 
  \begin{cases}
  s(X_i,U), \\
  \displaystyle\min_{X_j \in U} d(X_i,U-\{X_j\}).
  \end{cases}
\end{equation}
The optimal parents set of $X_i$ is a subset of $U$ that minimizes $d(X_i, U)$.
Then, the optimal choice of a sink minimizes the sum of scores of sub-networks consisting of the sink and the remaining nodes. If we denote an optimal score of a network over $U\subseteq\calX$ by $Q^*(U)$, then we have:
\begin{equation}\label{eq:Qopt}
Q^*(U) = \min_{X_i \in U}(d(X_i,U - \{X_i\}) + Q^*(U - \{X_i\})),
\end{equation}
and by using dynamic programming to compute $Q^*(\calX)$ we can construct an optimal ordering $\pi^*(\calX)$.

The dynamic programming algorithm can be visualized as operating on the lattice $L$ with $n + 1$ levels formed by the partial order ``set inclusion'' on the power set of $\calX$~\cite{Nikolova2013,Koivisto2004,Yuan2011} (see Figure~\ref{fig:path}). Two nodes in the lattice, $U'$ and $U$, are connected only if $U' \subset U$ and $|U| = |U'| + 1$. Here we use $U$ to denote both a subset of $\calX$ and the corresponding node in the lattice $L$. An important property of the lattice is that any path from its root to one of its nodes is equivalent to a specific ordering of variables in that node. Moreover, an edge $(U',U)$ corresponds to evaluating $d(U - U', U')$. For instance, the path marked in Figure~\ref{fig:path} represents ordering $\pi(\calX) = [X_3, X_2, X_4, X_1]$, and edge $(\{X_3\},\{X_2,X_3\})$ means computing $d(X_2,\{X_3\})$. In~\cite{Yuan2011} Yuan et al. observed that finding an optimal ordering (i.e. an optimal network structure) is equivalent to finding a shortest path from the root to the sink in the dynamic programming lattice (which they call an {\it order graph}). Because this formulation gives a significant flexibility in the design of search algorithms we decided to adopt it in our approach.

\begin{figure*}[!t]
  \centering
  \subfloat[]{\label{fig:path}\includegraphics[scale=0.7]{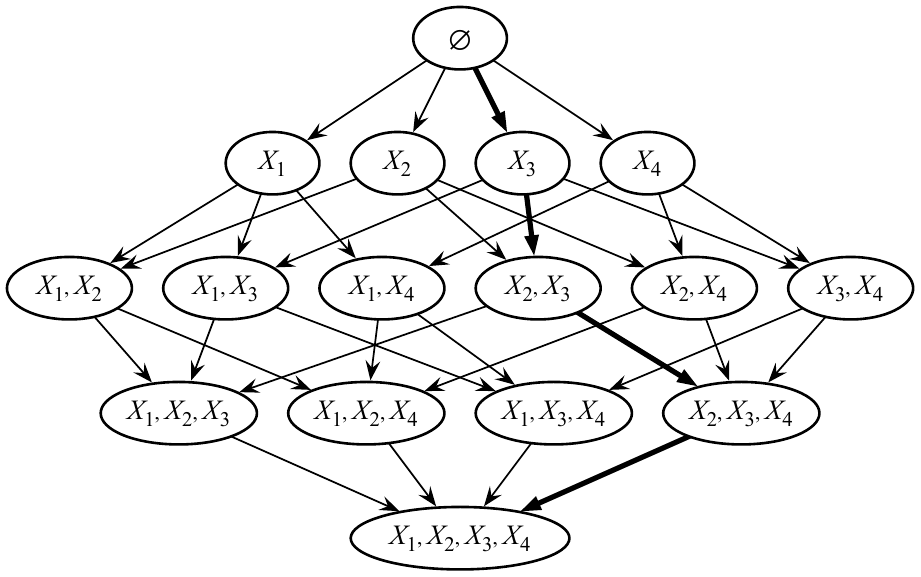}}~~
  \subfloat[]{\label{fig:spg}\includegraphics[scale=0.7]{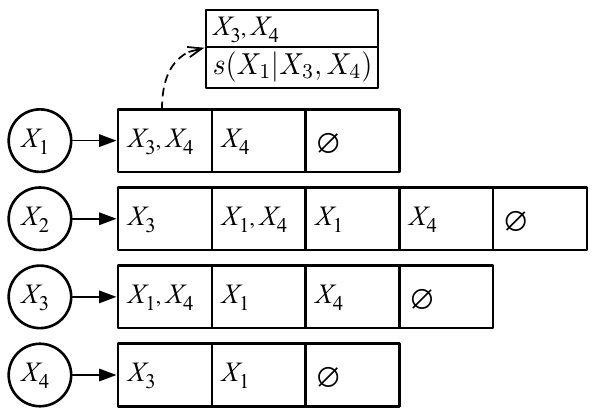}}~~
  \subfloat[]{\label{fig:path-compressed}\includegraphics[scale=0.7]{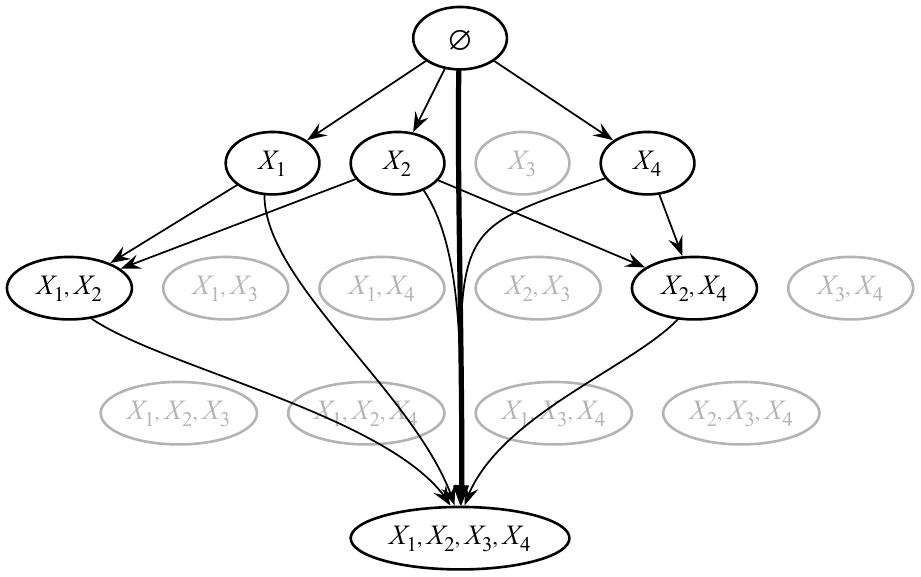}}  
 \caption{(a) Dynamic programming lattice for the problem with four variables, (b) example parent graph structure where for each variable $X_i$ an ordered vector of tuples $(U,s(X_i,U))$ is stored, and (c) the constrained lattice created via our optimal path extension technique as prescribed by the parent graph in (b). Let path marked in bold be the optimal solution. By constraining the dynamic programming lattice, after discovering node $\{X_3\}$ a search algorithm can follow directly to the final node.}\label{fig:pathnspg}
\end{figure*}

\section{Our Proposed Method}\label{sec:method}

The starting point for solving our optimization problem are recurrences in Equations~(\ref{eq:d})~and~(\ref{eq:Qopt}). The standard dynamic programming approach involves memoization of both $Q^{*}$ and $d$. If we imagine dynamic programming as progressing in the top-down manner over the lattice $L$, then the memory complexity of memoization is $\Theta\left(\binom{n}{\frac{n}{2}}\right)$, which is the number of nodes in the largest layer of $L$. This quickly becomes prohibitive for any but small number of variables. On the computational side, dynamic programming requires that all edges in the lattice are visited, which implies $\Theta\left(n \cdot 2^n\right)$ steps. By casting the problem into the shortest path formulation we gain the flexibility of considering both recurrences independently, such that the memory and time complexity are reduced.

To address the complexity of computing and storing $d$, we use the concept of {\it parent graph}~\cite{Yuan2011}. When computing $d$, for each variable $X_i$ we are expected to memoize the dynamic programming lattice over the power set of $\calX - X_i$. However, this process can be optimized as follows. We say that $U \subseteq \calX - \{X_i\}$ is a maximal candidate parents set for $X_i$ if no subset of $U$ has score better than $d(X_i, U)$, i.e. $\displaystyle\forall_{U' \subset U} d(X_i,U) < d(X_i,U')$ as we are considering the minimization problem. By definition, we have that $d(X_i,U) = s(X_i,U)$ for any maximal candidate parents set~$U$. As $U$ represents the best score for all its possible subsets, it is sufficient to memoize $s(X_i,U)$ only. Then, by storing $s$ for all maximal parents sets of $X_i$ we can answer efficiently all queries $d(X_i,U'')$ for any $U'' \supseteq U$. If $U''$ is one of the maximal parents sets of $X_i$ we simply return stored $s(X_i,U'')$. Otherwise, $d(X_i,U'')$ must be equal to the smallest $s$ among all maximal parents sets for which $U''$ is a superset.

Because the number of all maximal candidate parents sets is much smaller than the entire dynamic programming lattice, and they can be discovered incrementally, we drastically reduce the overall memory footprint. Depending on the scoring function and the input data the reduction might be by orders of magnitude (see for example Table~\ref{tab:data} in Section~\ref{sec:results}).

In our approach, to store and access all maximal candidate parents sets we create a parent graph data structure that for each $X_i$ maintains an ordered vector of tuples $(U,s(X_i,U))$ (see example in Figure~\ref{fig:spg}). Tuples are sorted in the ascending order of $s$, and we use binary encoding to represent $U$. The binary representation allows for $O(1)$ set containment and set equality checking as long as the number of variables does not exceed the word size of the executing hardware (e.g. $n\leq 64$ on a 64-bit architecture). By keeping vectors ordered, we can get the optimal choice of parents, and the corresponding score, for $X_i$ in $O(1)$, and we can answer arbitrary query $d(X_i,U)$ in $O(l)$, where $l$ is the size of the vector for $X_i$. This is because for a given $X_i$ its optimal parents set will be stored as the first entry of the corresponding sorted vector, and to answer $d(X_i,U)$ we have to find the first maximal parents set that is a subset~of~$U$.

To construct the parent graph for a given input data, multiple approaches and optimizations are possible, especially in how individual values $s$ are computed and how dynamic programming is executed. We discuss these techniques in a separate publication, and here we assume that the parent graph has been precomputed and can be accessed when searching for $Q^*$. We note however that in many cases the cost of parent graph construction is comparable or even more significant than the cost of computing $Q^*$ via the shortest path problem.

\subsection{Optimal Path Extension}

With the parent graph available we can focus now on the second recursion, Equation~(\ref{eq:Qopt}), to find an optimal network score and hence optimal network structure. As we already explained, the problem is equivalent to finding a shortest path from the root to the sink of the corresponding dynamic programming lattice $L$. As previously, the challenge is due to the immense size of $L$.

Let $Q(U, \pi)$, defined as: \[ Q(U,\pi)=\sum_{X_i \in U} d\big(X_i,\{X_j|X_j \prec X_i \textrm{ in } \pi(U) \}\big), \] be the score of a network over set $U \subseteq \calX$ prescribed by the topological ordering $\pi(U)$. Equivalently, $Q(U,\pi)$ is the length of the path from the root of the lattice $L$ to the node $U$ that yields ordering $\pi(U)$ (recall that we use $U$ to denote both a subset of $\calX$ and a node in $L$). To find the desired shortest path in $L$, and hence $Q^*$ and $\pi^*$, we could use any shortest path solver ranging from BFS to A-star and its variants (e.g. Iterative Deepening Search). However, in all cases memory constraint becomes a limiting factor. For example, in BFS at least two consecutive layers of $L$ have to be maintained in memory, and in A-star open and closed lists may grow excessively depending on the quality of the heuristic function used. Consequently, to scale up it is critical to further constraint the search space, i.e. reduce the number of nodes that have to be considered in the dynamic programming lattice. To achieve this, we introduce the {\it optimal path extension} technique.

Consider a node $U$ at the level $k$ in the lattice $L$ (the root of the lattice is at level $k=0$). This node has $k$ incoming edges and $n - k$ outgoing edges. Each of the outgoing edges corresponds to one particular way in which $U$, and thus any of its corresponding orderings/paths can be extended. However, in many cases we can immediately identify the only extensions that can lead to the optimal path from $U$ to the sink of the lattice. Since other extensions of $U$ will be suboptimal, we can safely remove them from consideration as they cannot be a part of the final shortest path. To identify a node that can be optimally extended we use the following observation. If $U$ is a superset of the optimal parents set of $X_i$, then by definition of $d$ no variable can be added to $U$ such that the score $d(X_i, U)$ is improved. Moreover, to maintain topological ordering $X_i$ must be preceded by all variables in $U$. Consequently, any optimal path from $U$ to the sink of the lattice must include edge from $U$ to $U \cup \{X_i\}$. This intuition is captured in the following theorem:

\begin{theorem}[Optimal Path Extension]\label{th:path} Let $U$ be a superset of the optimal parents set for $X_{i} \in \calX - \{X_i\}$. Then, in the optimal path from $U$ to the sink of the dynamic programming lattice $U$ must be followed by $U \cup \{X_i\}$.
\end{theorem}

\begin{proof}
Let $\pi_1(\calX) = \pi(U){\frown}X_i{\frown}X_j{\frown}\pi(V)$ and $\pi_2(\calX) = \pi(U){\frown}X_j{\frown}X_i{\frown}\pi(V)$, where $V = \calX - U - \{X_i,X_j\}$, represent two possible paths from the source to the sink of the dynamic programming lattice. We have that \[Q(\calX, \pi_1) = Q(U,\pi) + d(X_i,U) + d(X_j, U \cup \{X_i\}) + R\] and \[Q(\calX, \pi_2) = Q(U,\pi) + d(X_j,U) + d(X_i, U \cup \{X_j\}) + R, \] where $R$ is the length of the shortest path from $U \cup \{X_i,X_j\}$ to the sink of the lattice. Because $U$ is the superset of the optimal parents set for $X_i$ we have $d(X_i,U \cup \{X_j\}) = d(X_i,U)$. Now we consider two cases. If $U$ is not optimal parents set for $X_j$, we have $d(X_j,U) \geq d(X_j,U \cup \{X_i\})$ and it follows that:
\begin{align*}
  Q(\calX, \pi_2) & \geq Q(U,\pi) + d(X_j, U \cup \{X_i\}) + d(X_i, U) + R \\
  & \geq Q(\calX, \pi_1),
\end{align*}
and hence $\pi_2$ is not optimal. If $U$ is the optimal parents set for $X_j$, then both paths become equivalent and optimal.
\end{proof}

To better illustrate the optimal path extension idea, consider example dynamic programming lattice and parent graph presented in Figure~\ref{fig:spg}. The optimal parent set of $X_2$ and $X_4$ consists of $X_3$ only. Now take node $U=\{X_3\}$. Since $U$ is a superset of the optimal parent set of $X_2$ and $X_4$, from Theorem~\ref{th:path} to extend $U$ it is sufficient to consider one of only these two variables. Suppose that we extend $U$ by adding $X_2$. The new node $\{X_2,X_3\}$ with the ordering $[X_3,X_2]$ remains the superset of the optimal parent set of $X_4$. Thus, we can further extend $\{X_2,X_3\}$ by adding $X_4$ with the corresponding ordering $[X_3,X_2,X_4]$. In the final step, we can extend one more time by adding $X_1$, hence reaching the sink of the lattice. In a similar way, we can extend $\{X_1,X_4\}$ by including $X_3$ and then $X_2$. In some cases extension will not be possible. For example, nodes $\{X_1\}$ and $\{X_1,X_2\}$ cannot be extended as no variable has optimal parents set that would be a subset of either of them. If we consider all possible path extensions, then the final compacted lattice will be reduced by one node and seven edges (see Figure~\ref{fig:path-compressed}).

By applying our path extension technique we can significantly reduce the number of nodes and edges that have to be considered in the dynamic programming lattice. The extent to which reduction can be performed depends on the size of the optimal parents set of each variable -- smaller the optimal parents set, higher the chance that the optimal path extension can be applied. Moreover, the effectiveness of our method will be higher for larger problems (i.e. problems with larger $\calX$) because the dynamic programming lattice will include more nodes with a potential to extend. While at this moment we do not have complete theoretical bound on the expected number of nodes and edges that can be removed via the path extension technique, our experimental results in Section~\ref{sec:results} show excellent performance in practice.

\subsection{Searching via Optimal Path Extension}

The net effect of using our optimal path extension technique is compaction of the dynamic programming lattice. However, it would be counterproductive to first build the lattice and then apply the technique. Instead, the optimal path extension can be efficiently combined with any shortest path solver. To show how, we use the classic A-star search with a simple heuristic function. The function relaxes the BN acyclicity constraint and assumes that all variables not included in the currently explored node form a network by selecting optimal parents from among all other variables. Formally, the heuristic function is defined as $h(U) = \sum_{X_i \in \calX-U} d(X_i,\calX-X_i)$. The function is easy to implement and it is known to be consistent~\cite{Yuan2011}.

The resulting A-star procedure is outlined in Algorithm~\ref{alg:astar}. To represent a search state we use simple structure with attribute $.g$ storing the exact distance from the root to the current node in the lattice and $.h$ storing the estimated distance from the current node to the sink of the lattice (as given by function $h$). For convenience, we store also~$.f$, which is the sum of $.g$ and $.h$. The corresponding set of variables for a given state (i.e. the actual lattice node) is stored in attribute $.set$ using binary encoding as explained earlier. Finally, to reconstruct the optimal path and hence ordering we store also parent node information in attribute~$.p$. We note that the final optimal network structure can be easily reconstructed from the parent graph and the shortest path information.

The algorithm follows the classic A-star pattern with $Q$ representing open list implemented as Fibonacci heap, and $C$ maintaining a closed list implemented as a simple hash table. The search states corresponding to the explored lattice nodes are generated on the fly in the loop in line~\ref{ln:loop}. In lines~\ref{ln:dopt1},~\ref{ln:dsuperset}~and~\ref{ln:dopt2} we use the parent graph structure to extract values $d$. Recall that this requires a linear scan to obtain $d(X_i,v.set)$, and $O(1)$ to obtain $d(X_i, \calX - \{ X_i \})$, which is the score of the optimal parents set for~$X_i$. The cost of the linear scan in line~\ref{ln:dopt2} is in general negligible. This is because as the algorithm progresses the $v.set$ becomes larger and hence the probability of finding a relevant subset in the parent graph increases. The key element of the algorithm is the path extension procedure invoked in line~\ref{ln:path}. The procedure, outlined in Algorithm~\ref{alg:path}, returns a search state that can be reached directly from the current state via application of our path extension technique (so for example, for the node $\{X_3\}$ in Figure~\ref{fig:pathnspg} it would return node $\{X_1,X_2,X_3,X_4\}$).

The path extension procedure iteratively applies Theorem~\ref{th:path} to the input node represented by $u.set$. First, it tests each variable $X_i$ to see whether the input node is a superset of the $X_i$'s optimal parents set (line~\ref{ln:optsuper}). In practice, this requires one set containment check between $u.set$ and the set of variables stored in the first entry of the parent graph for $X_i$. If optimal path extension can be applied, state $u$ is updated and the process continues until no extension is possible. The final node is returned back to the main A-star procedure that follows without any changes.

\begin{algorithm}
  \caption{\textsc{A-star With Optimal Path Extension}}\label{alg:astar}
  \begin{algorithmic}[1]
    \STATE $s.g \leftarrow 0$
    \STATE $s.h \leftarrow 0$
	\FOR {$X_i \in \calX$}
    	\STATE $s.h \leftarrow s.h + d(X_i, \calX - \{ X_i \} )$\label{ln:dopt1}
    \ENDFOR
    \STATE $s.f \leftarrow s.h$
    \STATE $s.set \leftarrow \phi$
    \STATE $s.p \leftarrow \phi$
    \STATE $Q.push(s)$
    \WHILE {$Q \neq \phi$}
		\STATE $v \leftarrow Q.pop()$
        \STATE $C.push(v)$
        \IF {$v.set = \calX$}
        	\RETURN $\textrm{\textsc{Backtrack}}(v, C)$
        \ENDIF
        
        \FOR {$X_i \in \calX - v.set$}\label{ln:loop}
            \STATE $u.g \leftarrow v.g + d(X_i, v.set)$\label{ln:dsuperset}
            \STATE $u.h \leftarrow v.h - d(X_i, \calX - \{ X_i \})$\label{ln:dopt2}
            \STATE $u.f \leftarrow u.g + u.h$
            \STATE $u.set \leftarrow u.set \cup \{X_i\}$
            \STATE $u.p \leftarrow v.set$
            \STATE $ u \leftarrow \textrm{\textsc{PathExtension}}(u)$\label{ln:path}
           
            \IF {$u \notin C$}  
            	\IF {$u \in Q$}
            		\STATE $pu \leftarrow Q.handle(u)$
    	            \IF {$pu.f > u.f$}
	                	\STATE $pu \leftarrow u$
                	    \STATE $Q.update(pu)$
	                \ENDIF
    	        \ELSE
        	    	\STATE $Q.push(u)$
            	\ENDIF
        	\ENDIF
        \ENDFOR        
	\ENDWHILE
  \end{algorithmic}
\end{algorithm}

\begin{algorithm}
  \caption{\textsc{PathExtension}}\label{alg:path}
  \begin{algorithmic}[1]
    \REPEAT
    	\STATE $extended \leftarrow false$
    	\FOR {$X_i \in \calX - u.set$} 
    		\IF {$d(X_i, u.set) = d(X_i, \calX - \{ X_i \})$}\label{ln:optsuper}
				\STATE $u.g \leftarrow u.g + d(X_i, \calX - \{ X_i \})$
	            \STATE $u.h \leftarrow u.h - d(X_i, \calX - \{ X_i \})$
                \STATE $u.set \leftarrow u.set \cup \{X_i\}$ 
                \STATE $extended \leftarrow true$
        	\ENDIF
	    \ENDFOR
    \UNTIL{\NOT $extended$}
    \RETURN $u$
  \end{algorithmic}
\end{algorithm}

From the computational complexity perspective, our approach includes a minimal overhead (e.g. a $O(n)$ linear scan in Algorithm~\ref{alg:path}) at the benefit of significantly constraining the number of nodes that have to be considered and stored in $Q$ and $C$. From the implementation perspective, only one small procedure has to be added to the A-star core. This holds true for other search algorithms as well. For example, in case of BFS the path extension procedure could be invoked for every node before that node is pushed into the FIFO queue. Finally, the method does not conflict but rather complements other possible optimizations such as exploring independencies between variables, which we do not discuss or consider in this work.

\section{Experimental Results}\label{sec:results}

We implemented our proposed method in the \SABNA{} toolkit (Scalable Accelerated Bayesian Network Analytics). Currently, the toolkit supports efficient parent graph construction under the MDL scoring function from any categorical data, and different optimal search strategies. It is written in C++11 and is available under the MIT License from the GitLab repository (https://gitlab.com/SCoRe-Group/SABNA-Release).

To understand the performance characteristics of our approach, we compared it with a top-down Breadth First Search (BFS) and the A-star search as implemented in the \UrL{} package version from 2016-05-17~\cite{UrLearning2016}. We decided to use \UrL{} as this software provides some of the most advanced A-star search heuristics, and has been demonstrated to outperform other methods~\cite{Yuan2011}. All tools were compiled using GCC 4.9.2 with standard optimization flags. To perform our tests we used a dedicated Linux compute server running in the exclusive mode under the Simple Linux Utility for Resource Management (SLURM). The server has dual 10-core Intel Xeon E5v3 2.3GHz processor and 64GB of RAM. However, in all tests only a single core was used to run the tested code with the remaining cores left to the operating system.

\subsection{Test Data and Experimental Setup}

\begin{table}[t]
\caption{Datasets used in the experiments.}\label{tab:data}
\begin{tabularx}{\columnwidth}{XXXXl}
\toprule
Dataset     &   $n$ & $m$       &  PG size	& PG size reduction   	\\
\midrule
Mushroom    &   23  &   8,124  &   375,609		& $2.6 \times {10}^{2}  $	\\
Autos       &   26  &   159    &   2,391        & $3.6 \times {10}^{5} $ 	\\
Insurance   &   27  &   1,000  &   1,518        & $1.2 \times {10}^{6} $	\\
Water       &   32  &   1,000  &   328         	& $2.1 \times {10}^{8} $	\\
Soybean     &   36  &   266    &   5,926        & $2.1 \times {10}^{8} $	\\
Alarm       &   37  &   1,000  &   672         	& $3.8 \times {10}^{9} $	\\
Bands       &   39  &   277    &   887         	& $1.2 \times {10}^{10}$	\\
\bottomrule
\end{tabularx}
\end{table}

To perform our tests we used a collection of standard benchmark datasets summarized in Table~\ref{tab:data}~\cite{UrLearning2016}. To ensure that the results are comparable between \SABNA{} and \UrL{} we used the following protocol. Because \UrL{} supports only binary variables, all datasets had been transformed into the $\{0,1\}$ domain, with $0$ assigned to every value below the mean for a given variable and $1$ if the value was above the mean. The parent graph for each dataset had been constructed using the \UrL{} tool with default parameters, and all methods ran with the same parent graph. In Table~\ref{tab:data}, we report the size of each parent graph (PG) together with how its size is reduced compared to storing all values~$d$.
Finally, all tools were limited to the 64GB of RAM (i.e. no swap memory) and were terminated if they exceeded two hours runtime limit. In Tables~\ref{tab:runtime}--\ref{tab:memory} we summarize obtained results, with the runtime and memory usage averaged over 10 executions with a negligible variance. The runtime was measured via the system wall-clock, and approximate memory usage is based on the SLURM reports.

\begin{table}[t]
\caption{Runtime comparison of different methods.}\label{tab:runtime}
\begin{tabularx}{\columnwidth}{XXXX}
\toprule
Dataset 	& BFS 		 	& URLearning & \textbf{SABNA} \\
\midrule	
Mushroom    &   2m12s       &   1m29s   &   1m	    \\     
Autos       &   3m54s       &   37s     &   13s    \\
Insurance   &   8m14s     	   	&   7m25s   &   2m28s    \\
Water       &   M	     	&   M       &   2m8s     \\
Soybean     &   M           &   M       &   1h36m  \\
Alarm       &   M           &   T       &   1h3m  \\
Bands       &   M           &   T       &   1h10m    \\
\bottomrule
\end{tabularx}
\centering
~\\
M -- program ran out of memory. T -- program ran out of time.
\end{table}

\begin{table}[t]
\caption{Number of nodes visited by method ($\times 10^6$).}\label{tab:nodes}
\begin{tabularx}{\columnwidth}{XXXX}
\toprule
Dataset  	& Lattice size   	& URLearning & \textbf{SABNA} \\
\midrule
Mushroom	& 8.38    				&   2.50   &   2.72	    \\        
Autos       & 67.10					&   4.70   &   1.84	    \\
Insurance   & 134.21				&   57     &   13.52		\\
Water       & $4.29 \times 10^{3}$	&   --     &   13.03		\\
Soybean     & $6.87 \times 10^{4}$	&   --     &   330.50	\\
Alarm       & $1.37 \times 10^{5}$ 	&  --      &   217.37	\\
Bands       & $5.49 \times 10^{5}$ 	& --       &   233.80	\\   
\bottomrule
\end{tabularx}
\centering
~\\
BFS has to visit all $2^n$ nodes in the lattice.
\end{table}

\begin{table}[t]
\caption{Memory usage for different methods (in GB).}\label{tab:memory}
\begin{tabularx}{\columnwidth}{XXXX}
\toprule
Dataset 	& BFS 			& URLearning & \textbf{SABNA} \\
\midrule
Mushroom    &   0.23    	&   0.57     &   0.21       \\     
Autos       &   1.50  		&   1.84     &   0.001    	\\
Insurance   &   2.99    	&   10.67    &   1.07      	\\
Water       &   --	    	&   --       &   1.03      	\\
Soybean     &   --      	&   --       &   27.16    	\\
Alarm       &   --      	&   --       &   17.23      \\
Bands       &   --      	&   --       &   20.53     	\\
\bottomrule
\end{tabularx}
\end{table}

\subsection{Results Discussion}

We start the analysis by looking at the runtime of all three methods. Table~\ref{tab:runtime} shows that \SABNA{} is able to process all test datasets and it outperforms both BFS and the A-star strategy of \UrL{}, irrespective of the input dataset. As the number of variables increases, the performance difference becomes more pronounced, and for the largest dataset successfully processed by all methods (i.e. ``Insurance'') \SABNA{} is over three times faster than the other methods. Both \SABNA{} and \UrL{} use the A-star search strategy. However, \SABNA{} implements the most basic heuristic while \UrL{} employs a provably tighter heuristic with pattern database~\cite{Yuan2011,Felner2004}. In spite of this, \SABNA{} explores significantly fewer states of the dynamic programming lattice, as shown in Table~\ref{tab:nodes}. This implies that the efficiency of \SABNA{} should be attributed solely to our path extension technique. In fact, in additional tests not reported here, the basic A-star search performed only slightly better than BFS that visits all nodes in the lattice. This is significant considering how heuristic-sensitive is A-star. For example, because the \UrL{} heuristic is not well tuned to the data in ``Water'' and ``Soybean'' datasets, the open and closed lists of A-star explode and the method runs out of memory. By contrast, for the same datasets \SABNA{} requires only 1GB and 27GB respectively (see Table~\ref{tab:memory}). What is more, even for the largest datasets for which BFS and \UrL{} failed, \SABNA{} consumed no more than 28GB of memory. This is well explained by the results in Table~\ref{tab:nodes}. With one exception, \SABNA{} visits several times fewer nodes in the dynamic programming lattice than \UrL{}. This directly translates into a small memory footprint and clearly demonstrates the effectiveness of our approach. For the ``Mushroom'' dataset \SABNA{} explores more nodes, yet it remains faster. This can be attributed to the overhead due to the pattern database construction in \UrL. At the same time, the overhead of our method is minimal. Because the ``Mushroom'' dataset is relatively small the overhead becomes the major component of the overall runtime. To summarize, presented experimental results consistently demonstrate that our path extension technique significantly reduces the number of states that have to be explored during the search process. This has the effect of reducing both memory and computational complexity such that much larger data can be processed.

\section{Conclusion}\label{sec:conclusion}

In this paper, we presented a new approach to accelerate the exact structure learning of Bayesian networks. Our experimental results demonstrated that the method performs extremely well in practice, even though it does not improve the worst case complexity. Our method is flexible and can be seamlessly combined with different search strategies. One of the main challenges in finding optimal BN structures is exponentially growing space complexity. While our method partially addresses this challenge, it can be further improved by expanding into distributed memory architectures (e.g. similar to our previous work~\cite{Nikolova2013}). This could open new range of applications for exact structure learning, including in clinical decision support systems or in genetics where problems with large number of variables are common.

\section{Acknowledgments}\label{sec:acknowledgements}

Authors wish to acknowledge support provided by the Center for Computational Research at the University at Buffalo.

\bibliographystyle{IEEEtran}
\bibliography{references}

\end{document}